\documentclass[a4paper,10pt]{article}
\usepackage[utf8]{inputenc}
\usepackage[T1]{fontenc}

\usepackage{graphicx}
\usepackage{calc,amssymb,amsmath,amsfonts,amsthm}
\usepackage{subfig}
\usepackage{algorithm}
\usepackage{algpseudocodex}

\DeclareMathAlphabet{\mathmybb}{U}{bbold}{m}{n}
\newcommand{\1}{\mathmybb{1}}

\usepackage{authblk}

\usepackage[hidelinks]{hyperref}

\usepackage{orcidlink}

\DeclareFontFamily{U}{mathx}{}
\DeclareFontShape{U}{mathx}{m}{n}{<-> mathx10}{}
\DeclareSymbolFont{mathx}{U}{mathx}{m}{n}
\DeclareMathAccent{\widehat}{0}{mathx}{"70}
\DeclareMathAccent{\widecheck}{0}{mathx}{"71}

\usepackage[T1]{fontenc}

\newtheorem{definition}{Def.}[section]
\newtheorem{example}{Example}[section]
\newtheorem{theorem}{Theorem}[section]

\DeclareMathOperator{\Exp}{Exp}

\begin{document}

\title{\textbf{Simulation of Random {LR} Fuzzy Intervals}}

\author[1]{Maciej Romaniuk\orcidlink{0000-0001-9649-396X}\footnote{Corresponding author}}
\affil[1]{Systems Research Institute PAS, Newelska 6, 01-447 Warszawa, Poland\\
\emph{email:}~\href{mailto:mroman@ibspan.waw.pl}{mroman@ibspan.waw.pl}}

\author[2]{Abbas Parchami\orcidlink{0000-0002-0593-7324}}
\affil{Department of Statistics, Faculty of Mathematics and Computer, Shahid Bahonar University of Kerman, Kerman, Iran\\
\emph{email:}~\href{mailto:parchami@uk.ac.ir}{parchami@uk.ac.ir}}

\author[3,4]{Przemys{\l}aw Grzegorzewski\orcidlink{0000-0002-5191-4123}}
\affil[3]{Faculty of Mathematics and Information Science, Warsaw University of Technology, Koszykowa 75, 00-662 Warszawa, Poland\\
\emph{email:}~\href{mailto:przemyslaw.grzegorzewski@pw.edu.pl}{przemyslaw.grzegorzewski@pw.edu.pl}}
\affil[4]{Systems Research Institute PAS, Newelska 6, 01-447 Warszawa, Poland}

\maketitle

\begin{abstract}
Random fuzzy variables join the modeling of the impreciseness (due to their ``fuzzy part'') and randomness.
Statistical samples of such objects are widely used, and their direct, numerically effective generation is therefore necessary.
Usually, these samples consist of triangular or trapezoidal fuzzy numbers.
In this paper, we describe theoretical results and simulation algorithms for another family of fuzzy numbers -- LR fuzzy numbers with interval-valued cores.
Starting from a simulation perspective on the piecewise linear LR fuzzy numbers with the interval-valued cores, their limiting behavior is then considered.
This leads us to the numerically efficient algorithm for simulating a sample consisting of such fuzzy values.

\end{abstract}


\section{Introduction}
\label{intro}

Synthetically generated samples are essential tools for both real-life applications and more theoretically oriented approaches.
In the literature, many respective generation algorithms were described, including high-dimensional real-valued problems (see, e.g., \cite{10.5555/1051451}) and fuzzy approaches (see, e.g., \cite{Colubi2002,GONZALEZRODRIGUEZ2009642,Grzegorzewski_amcs2022,Parchami20243583}).

In the case of fuzzy numbers (FNs for short), the simulated output is usually a connection of two different viewpoints for data: impreciseness, modeled by fuzziness itself, and uncertainty, related to random phenomena.
In the literature, such objects are described by various notions of fuzzy random values (see, e.g., the ontic view in the Puri and Ralescu sense \cite{Puri} or the epistemic view related to the approach of Kruse and Kwakernaak \cite{Kruse1982,Kwakernaak} among many others \cite{Couso2014}).

Usually, samples of fuzzy random variables consist of triangular or trapezoidal FNs \cite{Grzegorzewski_ijcis2020,Grzegorzewski_amcs2020,Grzegorzewski_amcs2022,Grzegorzewski2024277,Lubiano_IJAR2017}.
Special software packages to generate them are also available \cite{fuzzyResampling,FuzzySimRespack}.
These types of FNs are easy to handle but sometimes too restrictive when plenty of other membership functions are still possible \cite{Grzegorzewski2008,Lubiano_FSS2017}.
Therefore, numerical methods should be introduced for other kinds of fuzzy values (e.g., piecewise FNs \cite{Coroianu2019}).
Some attempts were made in this area \cite{Colubi2002,GONZALEZRODRIGUEZ2009642,Parchami20243583}, but new ideas are still necessary.

In this paper, we continue the work started in \cite{Parchami20243583} and generalize the theoretical and simulation results presented there for the case of LR fuzzy numbers with the interval-valued cores (so-called LR fuzzy intervals, LRFIs for short).
We present LRFIs from the simulation perspective, which enables their simple generation based on the introduced numerical algorithms.
Additionally, some theoretical properties of LRFIs are also discussed.

This paper is organized as follows.
In Sect. \ref{prel}, some notions and definitions concerning fuzzy sets and random variables are recalled.
The simulation approach for the piecewise linear LRFIs is described in Sect. \ref{simapfopietr}.
Theoretical results and the practically oriented algorithm for the limiting behavior of such FNs are considered in Sect. \ref{limibe}.
They allow us to generate the whole statistical sample of these FNs as presented in Sect. \ref{seq:FRS}.
In Sect. \ref{con}, some concluding remarks are provided.

\section{Preliminaries} 
\label{prel} 

This section recalls some necessary definitions and notations concerning fuzzy numbers.
A more exhaustive introduction to this topic can be found in, e.g., \cite{DuboisPrade}. 
Next, some notations related to the probability density function (pdf) and cumulative distribution function (cdf) are also provided.

\begin{definition} 
A \textbf{fuzzy number} (abbreviated further by FN) is an imprecise value characterized by a mapping $\tilde{A}:\mathbb{R}\to [0,1]$, called a membership function, such that its $\alpha$-cut defined by
\begin{equation}
\tilde{A}_{\alpha}=\begin{cases}
\{x\in\mathbb{R}:\tilde{A}(x)\geqslant\alpha\} & \text{if}\quad \alpha\in (0,1], \\
cl\{x\in\mathbb{R}:\tilde{A}(x)>0\} & \text{if}\quad \alpha=0,
\end{cases} \label{eq_acut}
\end{equation}
is a nonempty compact interval for each $\alpha\in [0,1]$, where $cl$ denotes the closure.
\end{definition} 

Easily seen,  every FN is entirely described by both its membership function $\tilde{A}(x)$ and family of $\alpha$-cuts $\{\tilde{A}_{\alpha}\}_{\alpha\in [0,1]}$.
Two of these $\alpha$-cuts have special meaning: the \textbf{core} (given by $\tilde{A}_1=\mathrm{core}(\tilde{A})$),  which contains all values fully compatible with the concept modeled by $\tilde{A}$, and the \textbf{support} (i.e., $\tilde{A}_0=\mathrm{supp}(\tilde{A})$) for which values are compatible to some extent with the concept described by~$\tilde{A}$.
Further on, a family of all fuzzy numbers will be denoted by $\mathbb{F}(\mathbb{R})$.

In the literature, many kinds of membership functions are used for FNs.
A special and widely used type of FNs is known as \textbf{LR fuzzy numbers} (LRFNs) and is given by
\begin{equation} 
\tilde{A}(x)=
\begin{cases}
 0 & \text{if}\quad x < a_1,  \\
 L \left( \frac{x-a_1}{a_2 - a_1}\right) & \text{if}\quad a_1 \leqslant x < a_2 ,  \\
 1 & \text{if}\quad a_2 \leqslant x < a_3 , \\
 R \left( \frac{a_4 - x}{a_4 - a_3}\right) & \text{if}\quad a_3 \leqslant x < a_4 , \\
 0 & \text{if}\quad x \geq a_4,  
\end{cases}
\label{eq:LFfn}
\end{equation} 
where $L, R: [0,1] \rightarrow [0,1]$ are continuous and strictly increasing functions for which $L(0)=R(0)=0, L(1)=R(1)=1$, and $a_1,a_2,a_3,a_4\in\mathbb{R}$ such that $a_1\leqslant a_2\leqslant a_3\leqslant a_4$.
When $L$ and $R$ are linear functions, i.e.,
\begin{equation}
L(x)=\frac{x-a_1}{a_2 - a_1} , \quad R(x)=\frac{a_4 - x}{a_4 - a_3}
\end{equation}
then we get a \textbf{trapezoidal fuzzy number} (abbreviated further on as TPFN).
And if $a_2=a_3$, then we say about a \textbf{triangular fuzzy number} (TRFN).

However, the term LRFN is sometimes recently used only for FNs with a single-element core (i.e., $a_2=a_3$ in \eqref{eq:LFfn}), and the other case (i.e., $a_2\neq a_3$) is known as \textbf{fuzzy intervals}.
In the following, we use the abbreviations LRFN (in the case of the single-element core) and LRFI (for the other case) to distinguish these two types of FNs properly.

FNs are widely used in the literature to model imprecise data (e.g., the results of experiments that can not be precisely described). 
But the nature of some data is also random, so there are approaches that join both the impreciseness and random uncertainty in the single entity -- \textbf{random fuzzy sets}  (or \textbf{fuzzy random variables}).
An example of such an approach is given by the following definition (known as the ontic view, or fuzzy random variables in the Puri and Ralescu sense \cite{Puri}):

\begin{definition}
Let $\mathcal{K}_c(\mathbb{R})$ denote a family of all nonempty convex intervals in~$\mathbb{R}$.
Given a probability space $(\Omega,\mathcal{A},P)$, a mapping $\tilde{X} : \Omega \to\mathbb{F}(\mathbb{R})$ is said to be a \textbf{random fuzzy number} if for all $\alpha\in [0,1]$ the $\alpha$-level set-valued mapping $\tilde{X}_\alpha: \Omega\to\mathcal{K}_c(\mathbb{R})$ is a compact random interval, i.e. $\tilde{X}_\alpha$ is a Borel measurable mapping w.r.t. $\mathcal{A}$ and the Borel $\sigma$-field generated by the topology induced by the Hausdorff metric on $\mathcal{K}_c(\mathbb{R})$.
\end{definition}

There are also other approaches (see, e.g., \cite{Blanco2013,Colubi2001,Gonzalez2006,Hesamian2020,Kratschmer2001}).
Especially the epistemic view, where we consider the fuzzy objects modeling imperfectly perceived real-valued random variables, should be mentioned here \cite{Kruse1982,Kwakernaak}.

Let $F_Z$ be the cdf of some random variable $Z$, where $f_Z$ is its pdf.
Let us assume that $\mathrm{supp}(f_Z)=[0,\xi]$, where $\xi\in\mathbb{R}^+\cup\{+ \infty\}$.
Then for any $y \in (0,\xi)$, we have the truncated pdf $f_{Z|y}$ given by
\begin{equation}
    f_{Z|y}(x)=\frac{f_Z(x)}{F_Z(y)-F_Z(0)}\1 (0<x<y),
    \label{eq:trunc}
\end{equation}
where $\1  (.)$ is the indicator function, and $F_{Z|y}$ is its truncated cdf counterpart.

The empirical distribution function (edf) is defined by
\begin{equation}
    \widehat{F}_{\mathbb{X}}(x)=\frac{1}{n}\sum_{i=1}^n \1 (X_i\leqslant x) ,
\end{equation}
where $\mathbb{X}=(X_1,\ldots,X_n)$ is a sample from the cdf $F_X$.
From the Glivenko-Cantelli lemma, $\widehat{F}_{\mathbb{X}}(x)$ converges to $F(x)$ uniformly over $x\in\mathbb{R}$.
However, in the case of small samples, it may be useful to replace edf with its linear interpolation (iedf) counterpart, given by
\begin{equation}
    \widecheck{F}_{\mathbb{X}}(x) = w\widehat{F}_{\mathbb{X}}(x_{(i)}) + (1+w)\widehat{F}_{\mathbb{X}}(x_{(i+1)}), \quad\text{for}\quad x_{(i)}\leqslant x<x_{(i+1)},
\end{equation}
where $x_{(1)},\ldots,x_{(n)}$ denote order statistics from the sample $\mathbb{X}=(X_1,\ldots,X_n)$ and
\begin{equation}
    w=\frac{x-x_{(i)}}{x_{(i+1)}-x_{(i)}}.
\end{equation}
It can be shown that
\begin{equation}
    \widecheck{F}_{\mathbb{X}}(x) = \begin{cases}
        0  & \text{if}\quad x<x_{(0)},\\
        \frac{i}{n}+\frac{x-x_{(i)}}{n(x_{(i+1)}-x_{(i)})} & \text{if}\quad x_{(i)}\leqslant x<x_{(i+1)}, \; i=0,1,\ldots,n-1, \\
        1 & \text{if}\quad x_{(n)}\leqslant x.
    \end{cases}
\end{equation}

\section{Simulation approach for piecewise LRFIs} 
\label{simapfopietr}

In the following, we apply the approach considered in \cite{Parchami20243583} to generalize the results obtained there to the case of LRFIs.

When the synthetic fuzzy datasets are necessary to conduct some numerical experiments, TPFNs are usually generated using five independent random variables $O, C^l,C^r, S^l,S^r$, where $O$ is a random variable corresponding to the ``true'' population distribution (known as \emph{the original} when the epistemic view is considered), while $C^l,C^r$ create the core of a TPFN, and $S^l,S^r$ are used for modeling its support (see, e.g., \cite{Grzegorzewski_ijcis2020,Grzegorzewski_amcs2020,Grzegorzewski_amcs2022,Grzegorzewski2024277,fuzzyResampling}). 
More precisely, to obtain a TPFN given by a foursome $[a_1,a_2,a_3,a_4]$, we calculate
\begin{equation}
	\begin{split}
    &a_1 = O - S^l - C^l, \quad a_2 = O - C^l,  \\
    &a_3 = O + C^r,  \quad\quad\quad a_3 = O + C^r + S^r ,
	\end{split}
\end{equation}
where the respective random values are generated using some specified pdfs, i.e., $O \sim f_O, C^l \sim f_{C^l},C^r \sim f_{C^r}, S^l \sim f_{S^l},S^r \sim f_{S^r}$.

But this procedure can be easily extended for a $k$-knot piecewise linear fuzzy number \cite{Baez2012,Coroianu2019} with the interval-valued core (piecewise LRFI for short).

To obtain a new piecewise LRFIs $\tilde{X}_k$ (see also Algorithm \ref{alg2}), we start from its ``original'' using a random draw from $f_{O}$.
Next, we add the respective left and right increments of the core (i.e., two random values are generated from $f_{C^l}, f_{C^r}$), and the left and right spreads of its support (i.e., two random values from $f_{S^l}, f_{S^r}$ are used).
In this way, we get
\begin{align}
    \mathrm{core}(\tilde{X}_k) & := [O -C^l, O + C^r], \label{eq:Xcore} \\
    \mathrm{supp}(\tilde{X}_k) & := [O -C^l-S^l,O + C^r+S^r]. \label{eq:Xsupp}
\end{align}

Suppose that $s^l, s^r$ are the realizations of the above-mentioned random values for the spread (i.e., $S^l, S^r$).
Similarly, we have the realizations $o, c^l, c^r$.
During the second step, we generate $k$ knots for the left and right arms of $\tilde{X}_k$ using the respective truncated pdfs $f_{S^l|s^l}$ and $f_{S^r|s^r}$.
To do this, $k$ \emph{iid} random values $l_1,\ldots,l_k$ from $f_{S^l|s^l}$ are drawn and sorted in nondecreasing order.
Then, we have $l_{(1)} \leq \ldots \leq l_{(k)}$ and obtain the following knots for the left arm of $\tilde{X}_k$:
\begin{equation}
    \left( o - c^l-l_{(k-i+1)},\, \frac{i}{k+1} \right), \quad i=1,\ldots,k
    \label{eq:leftknot}
\end{equation}
(compare with \eqref{eq:Xsupp}).
Similarly, $r_1,\ldots,r_{k} $ are \emph{iid} generated  using $f_{S^r|s^r}$, and then  increasingly ordered into $r_{(1)},\ldots,r_{(k)}$. 
They form the knots of the right arm of~$\tilde{X}_k$:
\begin{equation}
	\left( o+c^r+r_{(i)},\,  \frac{i}{k+1} \right), \quad i=1,\ldots,k.
    \label{eq:rightknot}
\end{equation}

During the last step, these knots are piecewise sequentially joined using the linear functions to obtain the membership function of $\tilde{X}_k$. 
Easily seen, the generated piecewise LRFI can be directly characterized by the previously mentioned pdfs and $k$, so we can write 
\begin{equation}
	\tilde{X}_k = \left[f_O, f_{C^l}, f_{C^r}, f_{S^l},f_{S^r} \right]_k .
\label{eq:Xplrfn}
\end{equation}

\begin{algorithm}
\caption{Simulation of a piecewise LRFI}
\label{alg2}
\begin{algorithmic}
	\Require {The number of knots $k \geq 0$, random probability densities $f_{O}, f_{C^r}, f_{C^l}, f_{S^l},f_{S^r}$.}
	\Ensure {The membership function of $\tilde{X}_k$.}
	\State {Generate independently $O \sim f_{O}, C^l \sim f_{I},  C^r\sim f_{I}, S^l \sim f_{S^l}, S^r \sim f_{S^r}$.}
	\State {Generate $l_1, \ldots , l_{k} \stackrel{iid}{\sim} f_{S^l \vert s^l}$}
	\State {Sort increasingly $l_1, \ldots , l_{n_{k}}$ to obtain $l_{(1)}, \ldots , l_{(k)}$.}
	\For {$j=1$ to $k$}
		\State {Calculate the left-hand knots \eqref{eq:leftknot}.}
	\EndFor
	\State {Generate $r_1, \ldots , r_{k} \stackrel{iid}{\sim} f_{S^r \vert s^r}$.}
	\State {Sort increasingly $r_1, \ldots , r_k$ to obtain $r_{(1)}, \ldots , r_{(k)}$.}
	\For {$j=1$ to $k$}
		\State {Calculate the right-hand knots \eqref{eq:rightknot}.}
	\EndFor
	\State {Connect $o- c^l$, the knots \eqref{eq:leftknot}, and $o- c^l - s^l$ using the piecewise lines to obtain the left arm of $\tilde{X}_k$.}
	\State {Connect $o+ c^r$, the knots \eqref{eq:rightknot}, and $o+ c^r + s^r$ using the piecewise lines to obtain the right arm of $\tilde{X}_k$.}
\end{algorithmic}
\end{algorithm}

Now, we illustrate the procedure mentioned above with the following toy example.

\begin{example} 
\label{ExamA1} 
Let us simulate $\tilde{X}_k = \left[ N(1,2), U(0,1), U(0,1), \Exp (3),  \Exp (3) \right]_2$, where $N(\mu,\sigma)$ stands for the normal distribution with the mean $\mu$ and standard deviation $\sigma$, $\Exp (\lambda)$ -- the exponential distribution with the parameter $\lambda$, and $ U(0,1)$ -- the uniform distribution on the unit interval $[0,1]$.

Let us assume that during the first step, we independently generate $O=1.717, C^l = 0.11, C^r = 0.41, S^l = 0.057, S^r = 0.186$ using the respective pdfs (i.e., $O \sim N(1,2), C^l \sim U(0,1)$, etc.).
Then, the core of $\tilde{X}_k$ is equal to $[1.717-0.11,1.717+0.41]=[1.607,2.127]$ and its support to $[1.717-0.11-0.057, 1.717+0.41+0.186]=[1.55,2.313]$.

During the second step, we generate $k = 2$ random values for the left and right-hand knots, using $\Exp (3)$ truncated to the intervals $\left[0, s^l = 0.057 \right]$ and $\left[0, s^r = 0.186 \right]$, respectively.
Let us assume that for the left arm, we obtain $ l_1=0.028 $ and $ l_2=0.017 $, which leads to $ l_{(1)}=0.017 $, $ l_{(2)}=0.028 $.
Then, using the formula \eqref{eq:leftknot}, the respective knots can be calculated as $\left( 1.579, \ 0.333 \right)$, $\left( 1.607 - 0.017 = 1.59, \ 0.667 \right) $.
Similarly, for the right arm, we generate $ r_1=0.052 $ and $ r_2=0.156 $, so $ r_{(1)}=0.052 $ and $ r_{(2)}=0.156 $.
Using the formula~\eqref{eq:rightknot}, the new knots are equal to $\left( 2.283, \ 0.333 \right) , \left( 2.179, \ 0.667 \right)$.

Finally, the respective membership function of $\tilde{X}_k$ is calculated when the obtained knots are connected with the core and the support using the piecewise linear functions (see Fig. \ref{Fig1}).
\end{example}

\begin{figure}[htbp]
\begin{center} 
\includegraphics[width=0.6\textwidth]{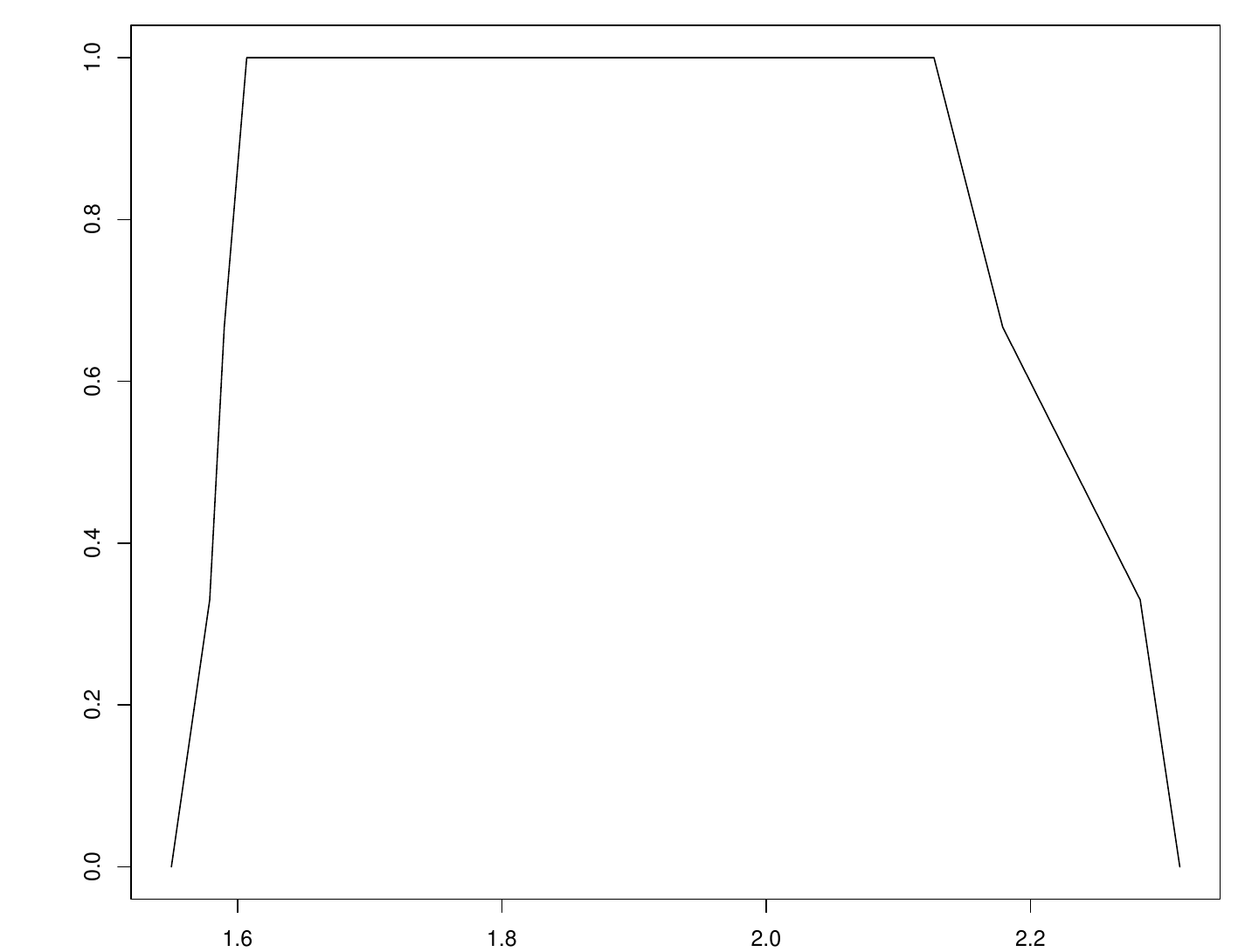}
\caption{The membership function of $\tilde{X}_k$ in Example \ref{ExamA1}.}
\label{Fig1} 
\end{center} 
\end{figure}

\section{Limiting behavior}
\label{limibe}

It is fruitful from the theoretical point of view to observe the behavior of the $k$-knot piecewise LRFI when $k \rightarrow \infty $ (see also \cite{Parchami20243583} for the similar results concerning LRFNs).
Then, let us suppose that we are interested in the generation of $\tilde{X}_k = [N(0,1), U(0,1),U(0,1),U(0,2),\Exp(1)]_k$ for the increasing number of the knots~$k$.
After fixing the core and support, we can observe that the left and right arm converged to the reliability functions (for the truncated uniform and exponential cdf, respectively) when $k \rightarrow \infty$ (see Fig.~\ref{ex2hist}).

\begin{figure}[htbp]
\centering
\subfloat{
  \includegraphics[width=0.49\textwidth]{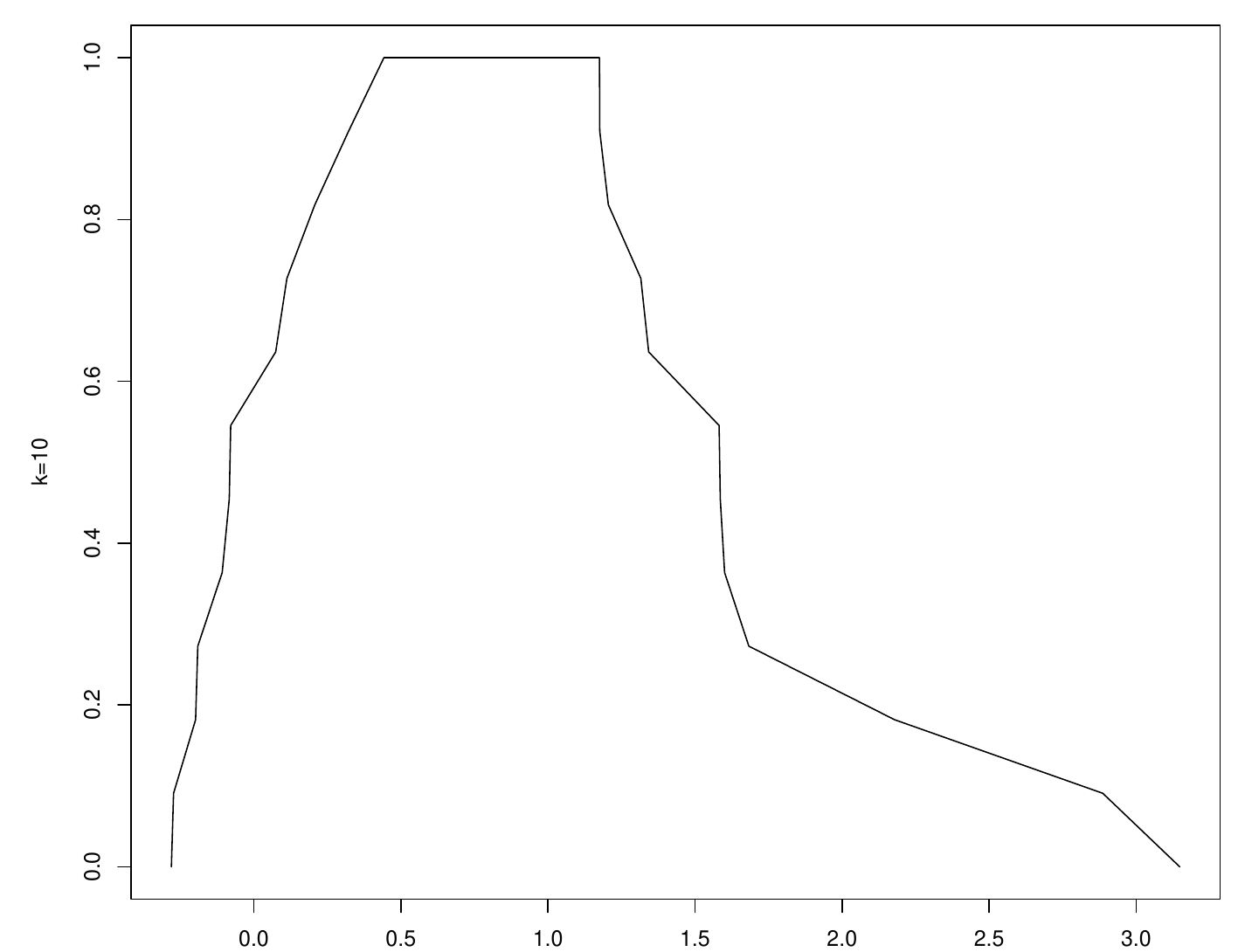}%
}\hfil
\subfloat{%
  \includegraphics[width=0.49\textwidth]{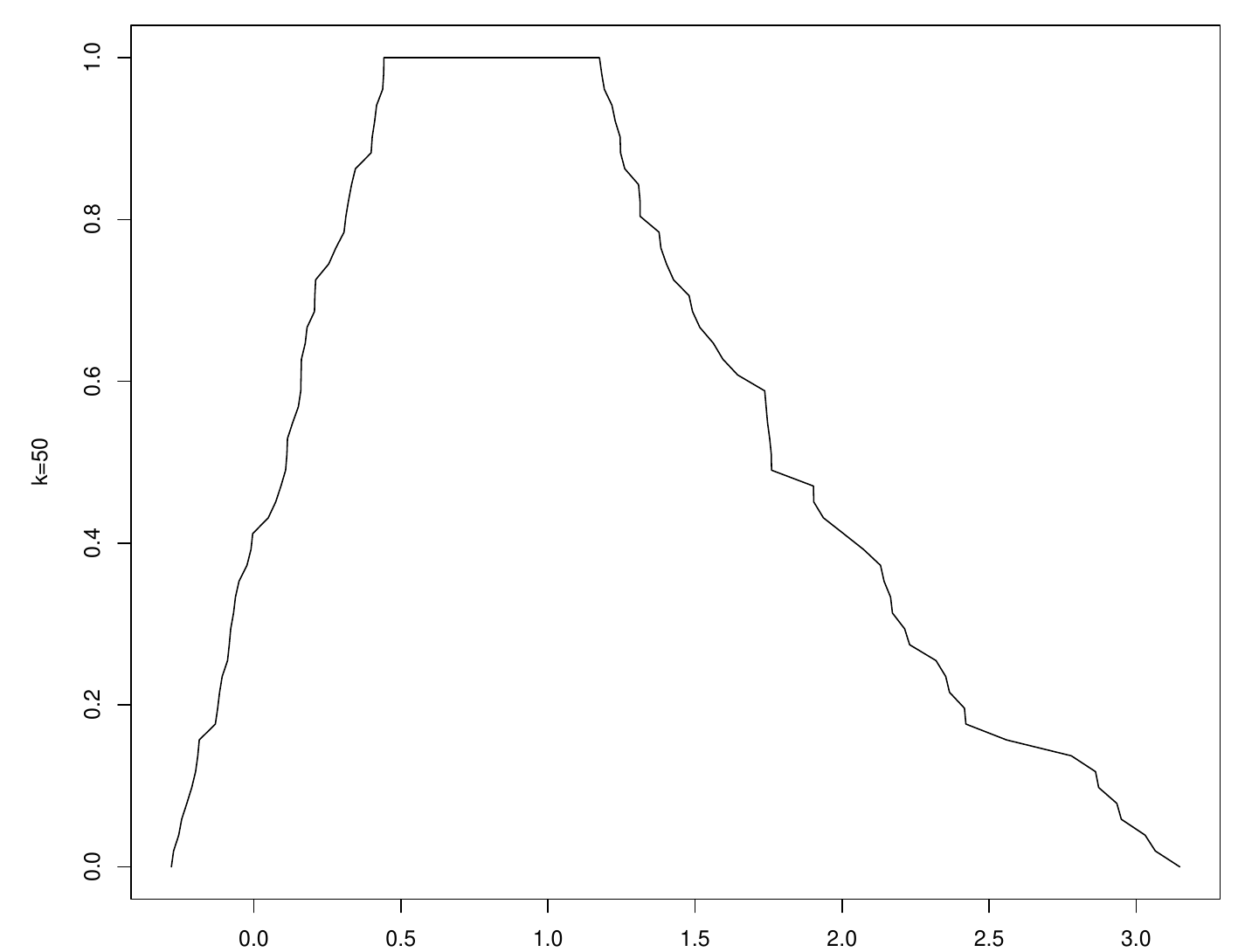}%
}

\subfloat{%
  \includegraphics[width=0.49\textwidth]{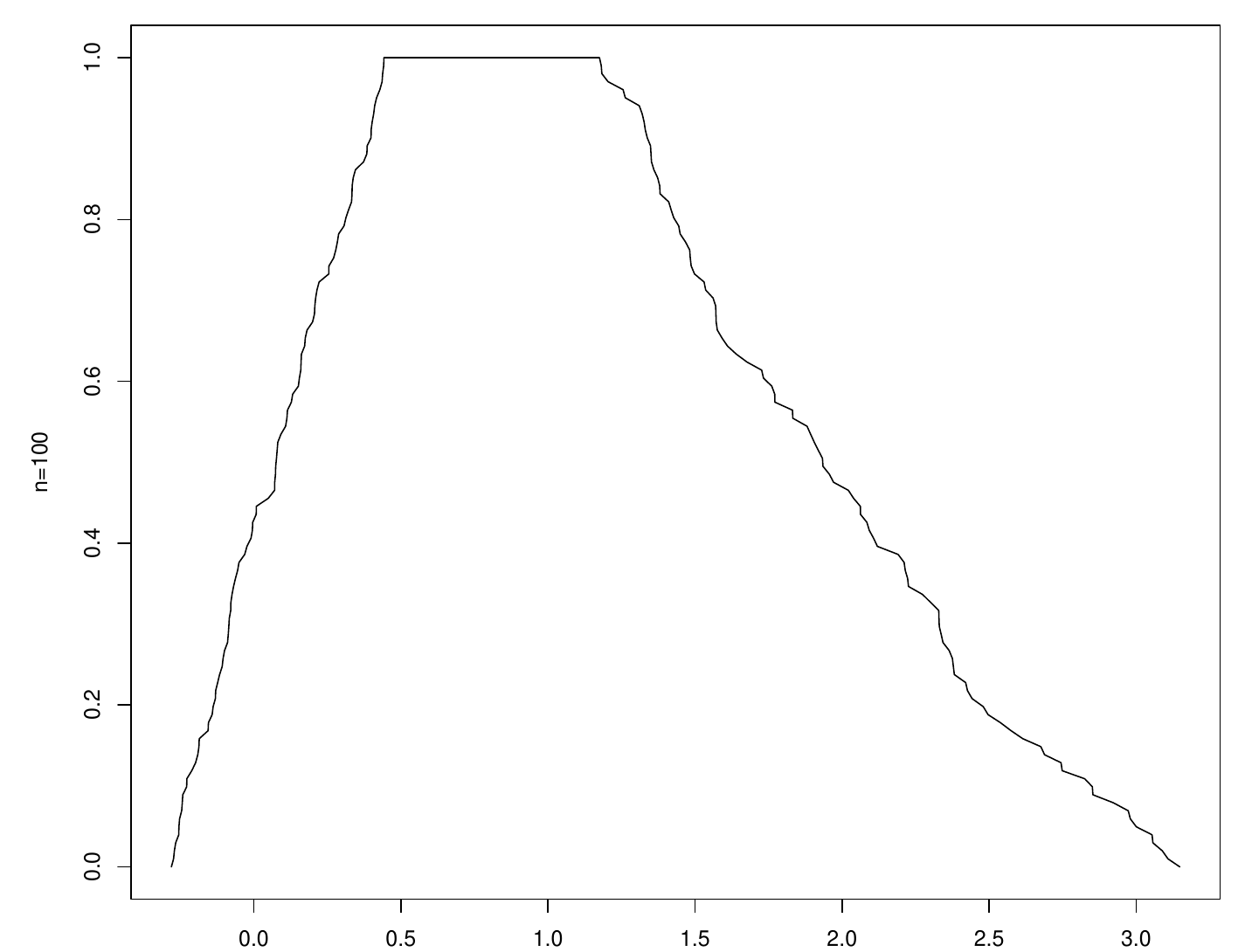}%
}\hfil
\subfloat{%
  \includegraphics[width=0.49\textwidth]{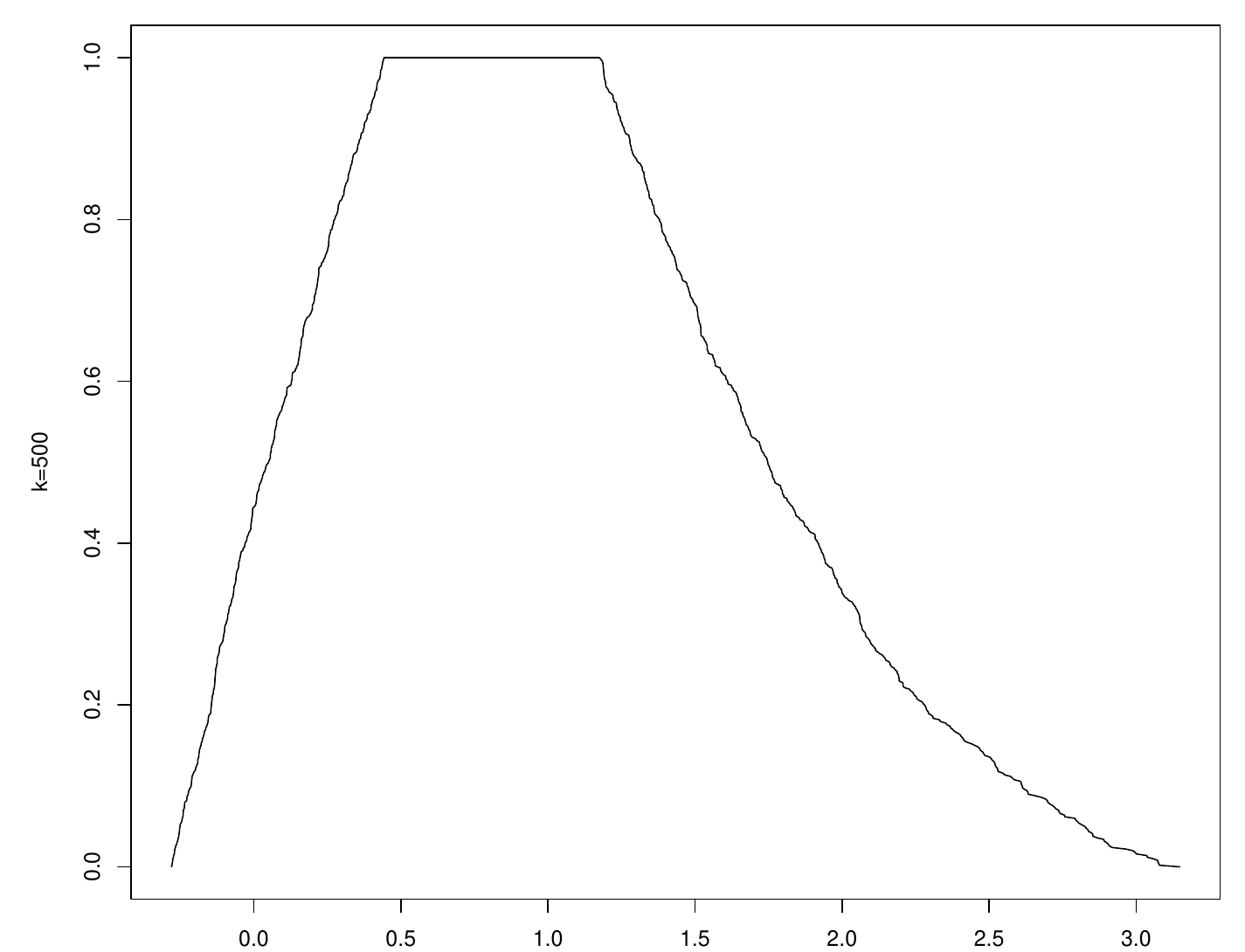}%
}

\caption{Limiting behavior of $\tilde{X}_k = [N(0,1), U(0,1),U(0,1),U(0,2),\Exp(1)]_k$ as a function of $k$.}
\label{ex2hist}

\end{figure}

This exemplary behavior results directly from the following theorem:

\begin{theorem} 
\label{Theorem1}
Let $\tilde{X}_k = \left[f_O, f_{C^l}, f_{C^r}, f_{S^l},f_{S^r} \right]_k$ be a $k$-knot piecewise LRFI.
Then for any fixed $x$, $\tilde{X}$ converges in probability to the random fuzzy number $\tilde{X}(x)$ as $k\to\infty$, where $\tilde{X}(x)$ has the following membership function 
\begin{equation} 
\tilde{X}(x) =  \begin{cases}
0 &  \text{if}\quad x\leqslant O-C^l-S^l, \\ 
1- F_{S^l|s^l}(O-C^l-x) &  \text{if}\quad O-C^l-S^l<x<O-C^l, \\
1 &  \text{if}\quad O-C^l \leq x \leq O+C^r, \\
1- F_{S^r|s^r}(x-(O + C^r)) &  \text{if}\quad O+C^r<x<O+C^r+S^r, \\ 
0 &   \text{if}\quad x\geqslant O+C^r +S^r. 
\end{cases}
\label{eq:Xlim}
\end{equation} 
\end{theorem} 

\begin{proof}
For the fixed number of the knots $k$, the membership function of $\tilde{X}_k$ is given by
\begin{equation} 
\tilde{X}_k(x) = 
\begin{cases}
0 & \text{if}\quad  x \leqslant O-C^l-S^l \\ 
1-\widecheck{F}_{l_1,\ldots,l_k}(O-C^l-x) &  \text{if}\quad  O-C^l-S^l<x<O-C^l \\ 
1 &  \text{if}\quad O-C^l \leq x \leq O+C^r \\ 
1- \widecheck{F}_{r_1,\ldots,r_k}(x-(O + C^r)) &  \text{if}\quad  O+C^r<x<O+C^r+S^r, \\ 
0 &  \text{if}\quad x \geqslant O+C^r +S^r,\\ 
\end{cases} 
\label{eq:X2} 
\end{equation}
where $\widecheck{F}_{l_1,\ldots,l_k}$ is the iedf based on the sample $(l_1,\ldots,l_k)$, and $\widecheck{F}_{r_1,\ldots,r_k}$ its counterpart for $(r_1,\ldots,r_k)$ (see Algorithm \ref{alg2} for the respective notation).
We know that the iedf $\widecheck{F}_{\mathbb{X}}(x)$ converges in probability to $F(x)$ for any fixed $x$.
For the considered $\tilde{X}_k$, we have the iedf $\widecheck{F}_{l_1,\ldots,l_k}$ for $(l_1,\ldots,l_k)$, and iedf $\widecheck{F}_{r_1,\ldots,r_k}$ for $(r_1,\ldots,r_k)$, respectively.
Then, from the Theorem of Large Numbers, we get $\lim_{k\to\infty} \widecheck{F}_{l_1,\ldots,l_k}(t) = F_{S^l|s^l}(t)$ and $\lim_{k\to\infty} \widecheck{F}_{r_1,\ldots,r_k}(t) = F_{S^r|s^r}(t)$ for any fixed $t$ with probability one.
Therefore, as stated in the theorem, $\tilde{X}_k(x)$ converges in probability to the random fuzzy number $\tilde{X}(x)$ for any fixed $x$ when $k\to\infty$.
\end{proof} 

The ``limit'' random fuzzy number $\tilde{X}$ from Theorem \ref{Theorem1} can be briefly described by the respective pdfs, so we get
\begin{equation}
    \tilde{X}=(O, C^r, C^l,S^l,S^r) .
    \label{eq:XlimLR}
\end{equation}
Its membership function \eqref{eq:Xlim} fulfills the requirements of \eqref{eq:LFfn} but is also related to the random values \eqref{eq:XlimLR}.
Therefore, we obtain a new object -- a \textbf{random LR fuzzy interval} (random LRFI).
Theorem \ref{Theorem1} also leads to the simulation procedure for such $\tilde{X}$ (see Algorithm \ref{alg10}).

\begin{algorithm}
\caption{Simulation of random LRFI}
\label{alg10}
\begin{algorithmic}
	\Require {Probability density functions $f_O,  f_{C^l},  f_{C^r}, f_{S^l}, f_{S^r}$.}
	\Ensure {The membership function of $\tilde{x}$.}
	\State {Generate independently $O\sim f_O, C^l \sim  f_{C^l}, C^r \sim  f_{C^r},S^l\sim f_{S^l}, S^r \sim f_{S^r}$.}
	\State {Calculate the truncated cdfs $F_{S^l|s^l}$ and  $F_{S^r|s^r}$ based on the truncated densities  $f_{S^l|s^l}$ and  $f_{S^r|s^r}$.}
	\State {Calculate the shape functions $L(x)= 1- F_{S^l|s^l}(o-c^l-x)$ for $o-c^l-c^l<x<o-c^l$, and $R(x)= 1- F_{S^r|s^r}(x-(o + c^r))$ for $o+c^r<x<o+c^r+s^r$.} 
\end{algorithmic}
\end{algorithm}

\section{Simulation of a fuzzy random sample}
\label{seq:FRS}

As mentioned previously, fuzzy random values are essential in many scientific applications.
Therefore, a simple and theoretically justified way for their generation is necessary.
Based on Theorem \ref{Theorem1} and Algorithm \ref{alg10}, we can simulate the whole \emph{iid} sample of random LRFIs (or simply -- the fuzzy random sample).
Therefore, Def. \ref{Def.Rand.Sample} leads directly to the respective procedure given by Algorithm \ref{alg3}.

\begin{definition} 
\label{Def.Rand.Sample} 
We say that random LRFIs $\tilde{X}_1,\ldots,\tilde{X}_n$ are \emph{iid} (or alternatively, form a simple fuzzy random sample) for the pdfs $f_O, f_{C^l},f_{C^r},f_{S^l},f_{S^r}$, if for each $i=1,\ldots,n$, $\tilde{X}_i=(O, C^r, C^l,S^l,S^r)$, and the respective random variables $O_1,\ldots,O_n$ are \emph{iid} generated using $f_O$, $C^l_1,\ldots,C^l_n$ are \emph{iid} generated using $f_{C^l}$ etc., and all these samples $O_1,\ldots,O_n, C^l_1,\ldots,C^l_n,\ldots, S^r_1, \ldots ,S^r_n$ are mutually independent.
\end{definition}

\begin{algorithm}
\caption{Simulation of the simple fuzzy random sample}
\label{alg3}
\begin{algorithmic}
	\Require {The sample size $n$, probability density functions $f_O, f_{C^l},f_{C^r},f_{S^l},f_{S^r}$.}
	\Ensure {The membership functions of $\tilde{x}_1,\ldots,\tilde{x}_n$.}
	\For{$i=1$ to $n$}
		\State {Use Algorithm \ref{alg10} to generate $\tilde{x}_i=(o_i,c^l_i, c^r_i ,s^l_i,s^r_i)$.}
		\State {Add this $\tilde{x}_i$ to the sample $\widetilde{\mathbf{x}}$.} 
	\EndFor
	
\end{algorithmic}
\end{algorithm} 

Example of the sample for $\tilde{X} = [N(0,1), U(0,1),U(0,1),U(0,2),\Exp(1)]$ generated using Algorithm \ref{alg3} can be found in Fig. \ref{Fig3}.

\begin{figure}[htbp]
\begin{center} 
\includegraphics[width=0.6\textwidth]{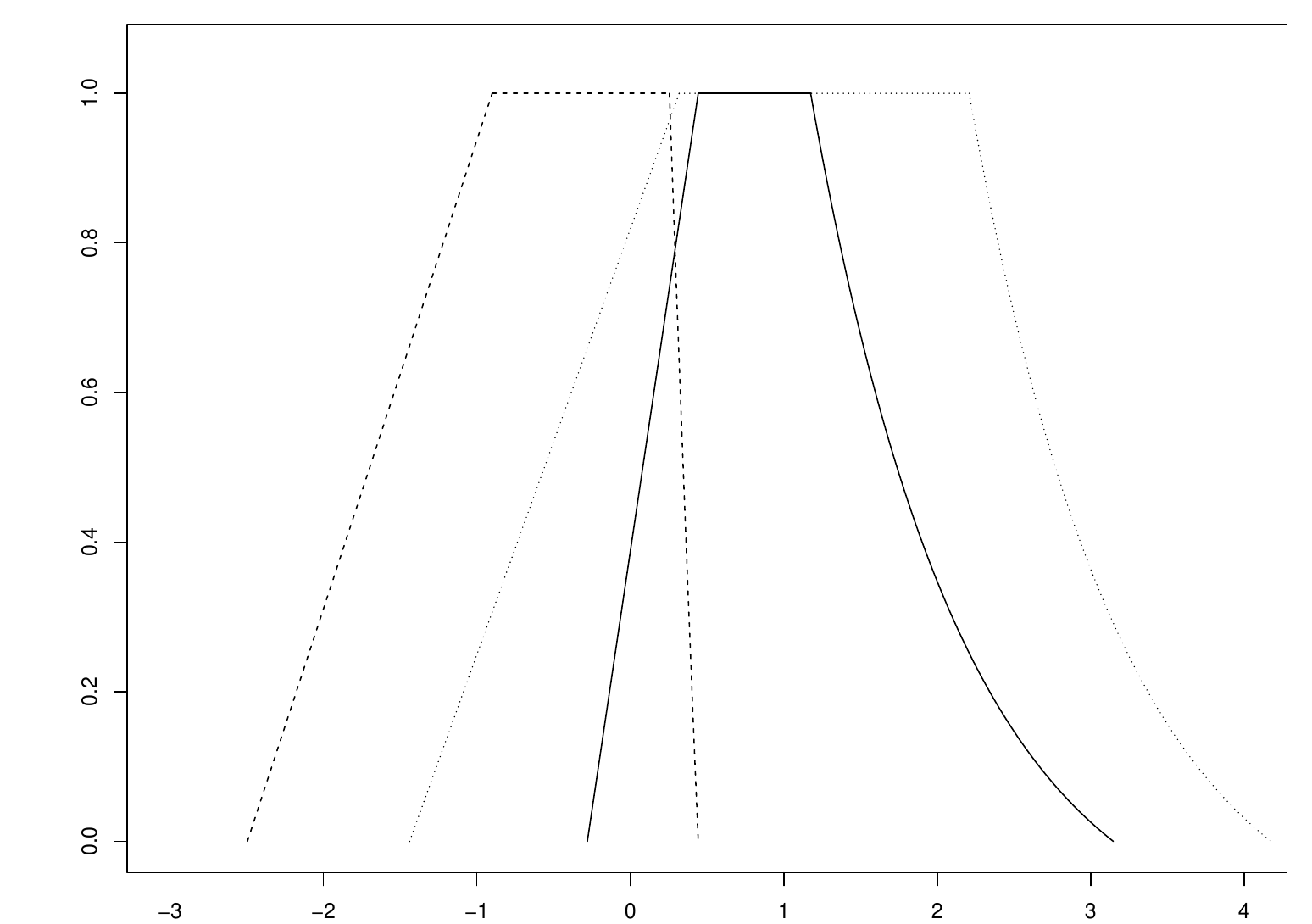}
\caption{Exemplary sample with $n=3$ elements generated for $\tilde{X} = [N(0,1), U(0,1),U(0,1),U(0,2),\Exp(1)]$.}
\label{Fig3} 
\end{center} 
\end{figure} 

\section{Conclusion}
\label{con}

In this paper, we propose the generalization of the approach presented in \cite{Parchami20243583} to the case of the LR fuzzy numbers with interval-valued cores.
The theoretical and numerical results are presented for such values, enabling their direct simulations.
These findings significantly broaden the family of the membership functions of FNs that can be used to generate synthetic statistical samples.
Until now, triangular or trapezoidal FNs have been used as models for simulation purposes.
Of course, these types of FNs are essential, but limiting yourself only to them may cause unexpected problems \cite{Grzegorzewski2008}.

Other approaches for different kinds of membership functions and types of fuzzy sets should be further investigated.
For example, the generation of samples of intuitionistic fuzzy sets may also be profitable in the case of some real-life applications.

\bibliographystyle{abbrv}
\bibliography{paper_IWIFSGN_24_bib1}

\begin{thebibliography}{10}

\bibitem{Baez2012}
A.~B\'{a}ez-S\'{a}nchez, A.~Moretti, and M.~Rojas-Medar.
\newblock On polygonal fuzzy sets and numbers.
\newblock {\em Fuzzy Sets and Systems}, 209:54--65, 2012.

\bibitem{Blanco2013}
A.~Blanco-Fernandez, M.~R. Casals, A.~Colubi, N.~Corral, M.~Garca-Barzana,
  M.~A. Gil, G.~Gonzalez-Rodriguez, M.~Lopez, M.~Montenegro, M.~A. Lubiano,
  A.~B. Ramos-Guajardo, S.~de~la Rosa~de Saa, and B.~Sinova.
\newblock Random fuzzy sets: a mathematical tool to develop statistical fuzzy
  data analysis.
\newblock {\em Iranian Journal of Fuzzy Systems}, 10(2):1--28, 2013.

\bibitem{Colubi2001}
A.~Colubi, J.~S. Domı\'{i}nguez-Menchero, M.~L\'{o}pez-D\'{i}az, and D.~A.
  Ralescu.
\newblock On the formalization of fuzzy random variables.
\newblock {\em Information Sciences}, 133(1):3--6, 2001.

\bibitem{Colubi2002}
A.~Colubi, C.~Fern\'{a}ndez-Garc\'{i}a, and M.~Gil.
\newblock Simulation of random fuzzy variables: An empirical approach to
  statistical/probabilistic studies with fuzzy experimental data.
\newblock {\em IEEE Transactions on Fuzzy Systems}, 10(3):384--390, 2002.

\bibitem{Coroianu2019}
L.~Coroianu, M.~Gagolewski, and P.~Grzegorzewski.
\newblock Piecewise linear approximation of fuzzy numbers: algorithms,
  arithmetic operations and stability of characteristics.
\newblock {\em Soft Computing}, 23(19):9491--9505, 2019.

\bibitem{Couso2014}
I.~Couso and D.~Dubois.
\newblock Statistical reasoning with set-valued information: Ontic vs.
  epistemic views.
\newblock {\em International Journal of Approximate Reasoning},
  55(7):1502--1518, 2014.

\bibitem{DuboisPrade}
D.~Dubois and H.~Prade.
\newblock {\em Fuzzy Sets and Systems: Theory and Applications}.
\newblock Academic Press, Boston, 1980.

\bibitem{GONZALEZRODRIGUEZ2009642}
G.~Gonz\'{a}lez-Rodr\'{i}guez, A.~Colubi, and W.~Trutschnig.
\newblock Simulation of fuzzy random variables.
\newblock {\em Information Sciences}, 179(5):642--653, 2009.

\bibitem{Gonzalez2006}
G.~Gonz\'{a}lez-Rodr\'{i}guez, M.~Montenegro, A.~Colubi, and M.~\'{A}ngeles
  Gil.
\newblock Bootstrap techniques and fuzzy random variables: Synergy in
  hypothesis testing with fuzzy data.
\newblock {\em Fuzzy Sets and Systems}, 157(19):2608--2613, 2006.

\bibitem{Grzegorzewski2008}
P.~Grzegorzewski.
\newblock Statistics with vague data and the robustness to data representation.
\newblock In D.~Dubois, M.~A. Lubiano, H.~Prade, M.~{\'A}. Gil,
  P.~Grzegorzewski, and O.~Hryniewicz, editors, {\em Soft Methods for Handling
  Variability and Imprecision}, pages 100--107, Berlin, Heidelberg, 2008.
  Springer.

\bibitem{Grzegorzewski_ijcis2020}
P.~Grzegorzewski, O.~Hryniewicz, and M.~Romaniuk.
\newblock Flexible bootstrap for fuzzy data based on the canonical
  representation.
\newblock {\em International Journal of Computational Intelligence Systems},
  13:1650--1662, 2020.

\bibitem{Grzegorzewski_amcs2020}
P.~Grzegorzewski, O.~Hryniewicz, and M.~Romaniuk.
\newblock Flexible resampling for fuzzy data.
\newblock {\em International Journal of Applied Mathematics and Computer
  Science}, 30:281--297, 2020.

\bibitem{Grzegorzewski_amcs2022}
P.~Grzegorzewski and M.~Romaniuk.
\newblock Bootstrap methods for epistemic data.
\newblock {\em International Journal of Applied Mathematics and Computer
  Science}, 32:288--297, 2022.

\bibitem{Grzegorzewski2024277}
P.~Grzegorzewski and M.~Romaniuk.
\newblock Bootstrapped tests for epistemic fuzzy data.
\newblock {\em International Journal of Applied Mathematics and Computer
  Science}, 34(2):277--289, 2024.

\bibitem{Hesamian2020}
G.~Hesamian, M.~G. Akbari, and J.~Zendehdel.
\newblock Location and scale fuzzy random variables.
\newblock {\em International Journal of Systems Science}, 51(2):229--241, 2020.

\bibitem{Kratschmer2001}
V.~Kr\"{a}tschmer.
\newblock A unified approach to fuzzy random variables.
\newblock {\em Fuzzy Sets and Systems}, 123(1):1--9, 2001.

\bibitem{Kruse1982}
R.~Kruse.
\newblock The strong law of large numbers for fuzzy random variables.
\newblock {\em Information Sciences}, 28(3):233--241, 1982.

\bibitem{Kwakernaak}
H.~Kwakernaak.
\newblock Fuzzy random variables, part {I}: Definitions and theorems.
\newblock {\em Information Sciences}, 15(1):1--15, 1978.

\bibitem{Lubiano_IJAR2017}
M.~A. Lubiano, A.~Salas, C.~Carleos, S.~{de la Rosa de S\'{a}a}, and
  M.~Ángeles Gil.
\newblock Hypothesis testing-based comparative analysis between rating scales
  for intrinsically imprecise data.
\newblock {\em International Journal of Approximate Reasoning}, 88:128--147,
  2017.

\bibitem{Lubiano_FSS2017}
M.~A. Lubiano, A.~Salas, and M.~A. Gil.
\newblock A hypothesis testing-based discussion on the sensitivity of means of
  fuzzy data with respect to data shape.
\newblock {\em Fuzzy Sets and Systems}, 328:54--69, 2017.

\bibitem{Parchami20243583}
A.~Parchami, P.~Grzegorzewski, and M.~Romaniuk.
\newblock Statistical simulations with {LR} random fuzzy numbers.
\newblock {\em Statistical Papers}, 65(6):3583--3600, 2024.

\bibitem{Puri}
M.~L. Puri and D.~A. Ralescu.
\newblock Fuzzy random variables.
\newblock {\em Journal of Mathematical Analysis and Applications},
  114(2):409--422, 1986.

\bibitem{10.5555/1051451}
C.~P. Robert and G.~Casella.
\newblock {\em Monte Carlo Statistical Methods}.
\newblock Springer-Verlag, Berlin, Heidelberg, 2005.

\bibitem{fuzzyResampling}
M.~Romaniuk and P.~Grzegorzewski.
\newblock Resampling fuzzy numbers with statistical applications:
  {F}uzzy{R}esampling package.
\newblock {\em R Journal}, 15(1):271--283, 2023.

\bibitem{FuzzySimRespack}
M.~Romaniuk, P.~Grzegorzewski, and A.~Parchami.
\newblock Fuzzy{S}im{R}es: Epistemic bootstrap -- the efficient tool for
  statistical inference based on imprecise data.
\newblock {\em R Journal}, 2024.
\newblock (in print).

\end{thebibliography}

\end{document}